\documentclass[runningheads,a4paper]{llncs}
\usepackage[T1]{fontenc}
\usepackage{graphicx}
\usepackage{booktabs}
\usepackage[misc]{ifsym}

\usepackage{amsmath,amssymb}
\usepackage{algorithm}
\usepackage{algpseudocode}
\newcommand{\norm}[1]{\lVert #1\rVert}
\newcommand{\Gam}{\Gamma}
\newcommand{\R}{\mathbb{R}}
\DeclareMathOperator*{\argmax}{arg\,max}

\newcommand{\LB}{\mathrm{LB}}
\newcommand{\UB}{\mathrm{UB}}
\newcommand{\rhocert}{\rho_{\mathrm{cert}}}
\newcommand{\rhoacq}{\rho_{\mathrm{acq}}}
\newcommand{\Cov}{\mathrm{Cov}}
\newcommand{\Risk}{\mathrm{Risk}}
\newcommand{\ind}{\mathbf{1}}
\usepackage{multirow}
\newcommand{\COVset}{\mathcal{C}}
\newtheorem{assumption}{Assumption}

\usepackage{subcaption} 

\begin{document}

\title{Selective Prediction from Agreement: A Lipschitz-Consistent Version Space Approach}

\titlerunning{Selective Prediction from Agreement}

\author{Mohamadsadegh Khosravani}
\authorrunning{M. Khosravani}
\institute{University of Regina}
\maketitle              

\begin{abstract}
We consider selective classification with abstention in the fixed-pool (or transductive) setting, where the unlabeled pool is given beforehand and only a subset of points can be queried for labels.

Our main insight is to view selective prediction through agreement: given queried labels and Lipschitz margin constraints in an embedding space, the version space of Lipschitz-consistent classification heads is well defined.

We obtain upper and lower Lipschitz margin bounds that define, for each pool point, a set of certified valid labels containing the prediction of every head in the version space. The model therefore predicts only when the label is forced (i.e., all consistent heads agree), and abstains otherwise.
We also propose a monotone submodular geometric proxy for budgeted querying, and show that a greedy algorithm retains the standard $(1-1/e)$ approximation factor.

\end{abstract}

\section{Introduction}\label{sec:intro}

Selective classification studies the risk--coverage trade-off: a classifier may abstain on some inputs to reduce error on the rest~\cite{chow1970reject,elyaniv2010selective}. In modern deep learning, abstention is often implemented via confidence thresholding or by learning an explicit selection function, and evaluated through risk--coverage curves~\cite{geifman2017selective,geifman2019selectivenet}. In contrast, we treat abstention as lack of forced agreement: we abstain exactly when the queried labels and the assumptions we impose are insufficient to determine a unique label.

We focus on the \emph{fixed-pool} (transductive) setting, where the unlabeled pool is given in advance, only a subset can be queried for labels~\cite{kottke2022stopping}, and the goal is to label (or abstain on) the pool elements~\cite{chapelle2006ssl,goldwasser2020beyond,kalai2021towards,zhu2003harmonic,zhou2004localglobal}. Our aim is therefore \emph{certification}: identify which pool points have labels that are uniquely determined under the assumed constraints, and abstain otherwise.

\paragraph{Agreement-based selective prediction.}
After observing labels on a subset $S$, we predict at a pool point $u$ only if the label is \emph{forced} by labeled set S, and  constraints, meaning that \emph{every} hypothesis consistent with the constraints predicts the same class at $u$. Concretely, we work in an embedding space induced by a fixed representation $g$ and consider classification heads whose one-vs-rest margin functions are Lipschitz in this space. We take the Lipschitz constants used for certification to be computable upper bounds that can be controlled during head training (e.g., via norm or spectral constraints~\cite{bartlett2017spectral,miyato2018spectral,tsuzuku2018lmt}). Together with center-margin constraints at queried points, these assumptions define a Lipschitz-consistent version space $\mathcal F(S)$.

Our certificate is two-sided: labeled class-$c$ centers propagate a certified \emph{lower} envelope (positive evidence) for the margin $M_c$, while labeled non-$c$ centers propagate a certified \emph{upper} envelope (negative evidence) for $M_c$. From these envelopes we compute, for each pool point $u$, a certified feasible label set $\Gamma(u)$ that contains the prediction of every head in $\mathcal F(S)$. We therefore predict only when $\Gamma(u)$ is a singleton; otherwise we abstain. Beyond singleton feasibility, we can also force a label by separation: if one class’s certified lower envelope exceeds every other class’s certified upper envelope by a gap, the label is uniquely determined.

\paragraph{Closure and budgeted querying.}
We prove a closure property: adding forced pseudo-labels as anchors does not further shrink the Lipschitz-consistent version space. Under a finite labeling budget, we connect querying to certification by introducing a geometric, label-free proxy: subject to a mild margin-floor assumption, it reduces to a ball-union coverage objective in embedding space. This proxy is monotone submodular, so greedy selection achieves the classical $(1-1/e)$ approximation factor~\cite{nemhauser1978analysis}.

\paragraph{Agreement vs.\ correctness.}
Our guarantees certify agreement over $\mathcal F(S)$ on the fixed pool; conditions under which agreement implies correctness are discussed in the supplementary material.

\paragraph{Positioning with prior work.}
Our work differs from standard selective classification, which typically abstains based on confidence or a learned selection head~\cite{chow1970reject,elyaniv2010selective,geifman2017selective,geifman2019selectivenet}: we abstain because the label is not identifiable under explicit constraints. The agreement / version-space viewpoint is classical in learning theory and is related in spirit to disagreement-based ideas in active learning; we instantiate it on a fixed pool via an explicit Lipschitz-consistent class and computable two-sided margin envelopes. The resulting feasible label sets $\Gamma(u)$ are set-valued, but unlike conformal prediction~\cite{vovk2005algorithmic,romano2020classification,angelopoulos2021uncertaintysets}, they encode labels consistent with a constraint-defined hypothesis class on a fixed pool rather than distribution-free marginal coverage. Finally, while Lipschitz control is often used for robustness certification, we use it here to define the certifying hypothesis class and derive sound force-or-abstain rules~\cite{bartlett2017spectral,miyato2018spectral,tsuzuku2018lmt}.

\paragraph{Contributions.}
We view selective prediction as agreement over a Lipschitz-consistent version space (Definitions~\ref{def:forced} and~\ref{def:consistent}).
We give a sound force-or-abstain certificate using two-sided class-wise margin bounds and forcing rules (Sections~\ref{sec:cert} and~\ref{sec:setvalued}),
prove a closure property for adding forced pseudo-labels (Section~\ref{sec:practical}),
and propose a geometric proxy for budgeted querying that is monotone submodular, so greedy achieves the standard $(1-1/e)$ approximation guarantee (Sections~\ref{sec:coverage} and~\ref{sec:submod}).
In experiments on transductive CIFAR-10 and SVHN test pools, we obtain \emph{non-trivial} certified coverage with small label budgets and show that proxy-based acquisition can expand the certified region (Section~\ref{sec:experiments}).
Our guarantees certify agreement (identifiability) within the hypothesis class.

\section{Setting, Notation, and Objectives}\label{sec:setting}

Let $D=\{x_i\}_{i=1}^N \subset \mathcal X$ be an unlabeled pool with index set $[N]:=\{1,\dots,N\}$.
Querying $i\in[N]$ reveals an oracle label $y_i\in\{1,\dots,C\}$.
Let $S\subseteq[N]$ be the queried (labeled) indices and $k:=|S|$ the labeling budget.

We use a representation--head decomposition: a representation $g:\mathcal X\to\mathbb R^m$ and a classification head
$f:\mathbb R^m\to\mathbb R^C$ producing logits $f(z)=(f_1(z),\dots,f_C(z))$.
For pool point $i$, let $z_i:=g(x_i)$ and $\hat y(z):=\argmax_{c} f_c(z)$.

\begin{remark}[Certificates depend on the representation]
All certificates are in the embedding space of the current representation $g$.
If $g$ changes, then the embeddings $\{z_i\}$ and the Lipschitz bounds used for certification must be recomputed.
\end{remark}

\subsection{Selective Prediction Metrics on a Fixed Pool}\label{sec:metrics}
We use the standard selective-classification risk--coverage formalism (e.g., \cite{chow1970reject,elyaniv2010selective}).
\begin{definition}[Selective predictor; coverage and selective risk]\label{def:selective}
A selective predictor is a mapping $\pi:[N]\to \{1,\dots,C\}\cup\{\bot\}$.
Let $y^\star(u)$ denote the evaluation-only pool label.
Define coverage
\[
\Cov(\pi) := \frac{1}{N}\sum_{u\in[N]} \ind\{\pi(u)\neq \bot\},
\]
and selective risk (conditional error rate)
\[
\Risk(\pi) :=
\frac{\sum_{u\in[N]} \ind\{\pi(u)\neq \bot\}\ind\{\pi(u)\neq y^\star(u)\}}
{\sum_{u\in[N]} \ind\{\pi(u)\neq \bot\}}.
\]
\end{definition}
These quantities follow the standard selective prediction / reject-option setup \cite{chow1970reject,elyaniv2010selective}. The metrics above use $y^\star(\cdot)$ only for evaluation.

\begin{definition}[Forced label (agreement over a version space)]\label{def:forced}
Fix a labeled set $S$ and a hypothesis class $\mathcal F(S)$.
A class $c^\star$ is \emph{forced} at pool index $u$ if for every $\tilde f\in\mathcal F(S)$ and every
$\hat y_{\tilde f}(u)\in\argmax_c \tilde f_c(z_u)$ we have $\hat y_{\tilde f}(u)=c^\star$.
Otherwise we abstain.
\end{definition}

\subsection{Certified Coverage Objective}\label{sec:objective}
Given a labeled set $S$, our method outputs a selective predictor $\pi_S$ on the fixed pool.
The goal is simple: \emph{predict on as many pool points as possible}, while keeping the error rate on the predicted points low.
We measure this with coverage $\Cov(\pi_S)$ and selective risk $\Risk(\pi_S)$ (Definition~\ref{def:selective}).

We consider two tasks: (i) \textbf{certification}, where $S$ is given and we compute $\pi_S$; and
(ii) \textbf{budgeted acquisition}, where we choose $S$ under a labeling budget to improve the risk--coverage trade-off.



\section{Certificates for Forced Labels}\label{sec:certificates}

\subsection{Margins and the Lipschitz Head Assumption}\label{sec:margins}

In this section, we define one-vs-rest margins of the classification head in embedding space.

\begin{definition}[One-vs-rest margin]\label{def:ovr}
For class $c\in\{1,\dots,C\}$, define
\[
M_c(z) := f_c(z)-\max_{k\neq c} f_k(z).
\]

\end{definition}

\begin{assumption}[Lipschitz margins in embedding space]\label{ass:lipschitz}
For each class $c$, the margin function $M_c:\R^m\to\R$ is $L_{M,c}$-Lipschitz with respect to $\ell_2$:
\[
\bigl|M_c(z)-M_c(z')\bigr| \ \le\ L_{M,c}\,\norm{z-z'}\qquad \forall z,z'\in\R^m.
\]
One can take a shared constant $L_{M,c}\equiv L_M$ across classes.
\end{assumption}

\begin{remark}[Obtaining $L_{M,c}$]\label{rem:lipschitz_practice}
We require that the constants $\{L_{M,c}\}$ are \emph{upper bounds} on the true Lipschitz constants of the margins.
In practice we use computable bounds from the head parameters: for $f(z)=Wz+b$, one may take
$L_{M,c}=\max_{k\neq c}\|w_c-w_k\|_2$; for deeper heads we use standard spectral/norm-based upper bounds
\cite{bartlett2017spectral,miyato2018spectral,tsuzuku2018lmt} (details in Supplementary).
\end{remark}

Next, we define center-margin constraints at queried points and the resulting Lipschitz-consistent version space $F(S)$.

\subsection{Center Margins and Lipschitz-consistent Class}\label{sec:consistency}

Fix a queried labeled subset $S\subseteq [N]$ with oracle labels $\{y_i\}_{i\in S}$.
To use these labels for certification, we associate each labeled point with a certified lower bound on the margin of its labeled class.

\begin{definition}[Center margin lower bound]\label{def:center_margin}
For each labeled index $i\in S$ with label $y_i$, let $\underline m_i\in\R_{\ge 0}$ satisfy
\[
\underline m_i \ \le\ M_{y_i}(z_i).
\]
We call $\underline m_i$ the \emph{center margin lower bound} at $i$.
\end{definition}

The next step is to formalize the hypothesis class induced by the margin constraints.

\begin{definition}[Lipschitz-consistent head class]\label{def:consistent}
Fix $S$ and center margin lower bounds $\{\underline m_i\}_{i\in S}$.
Define $\mathcal{F}(S)$ as the set of heads $\tilde f:\R^m\to\R^C$ whose induced margins $\tilde M_c$ satisfy:
\begin{enumerate}
  \item \textbf{(Lipschitz)} each $\tilde M_c$ is $L_{M,c}$-Lipschitz as in Assumption~\ref{ass:lipschitz};
  \item \textbf{(Center constraints)} for all $i\in S$, $\tilde M_{y_i}(z_i)\ge \underline m_i$.
\end{enumerate}
\end{definition}

\begin{remark}[Non-emptiness]\label{rem:nonempty}
By construction of $\{\underline m_i\}$ and $\{L_{M,c}\}$, the trained head belongs to $\mathcal F(S)$; in particular, $\mathcal F(S)\neq\emptyset$.
\end{remark}

All guarantees are uniform over $\mathcal F(S)$: we output a label only when every $\tilde f\in\mathcal F(S)$ predicts the same class at $u$; otherwise we abstain.
\subsection{Two-Sided Certified Margin Propagation}\label{sec:cert}
In this section, we derive class-wise certified envelopes for the margins at each pool point. Concretely, for every class $c$ and pool index $u$, we construct bounds
\[
\LB_c(u)\ \le\ \tilde M_c(z_u)\ \le\ \UB_c(u)\qquad \forall \tilde f\in\mathcal{F}(S).
\]
Intuitively, $\LB_c$ propagates certified positive evidence for class $c$ from labeled class-$c$ centers, while $\UB_c$ propagates certified negative evidence against class $c$ from labeled non-$c$ centers. We define these envelopes next.

\begin{definition}[Lower envelope]\label{def:lb}
For any pool index $u$ and class $c$, define
\[
\LB_c(u):=
\sup_{i\in S:\ y_i=c}\Bigl(\underline m_i - L_{M,c}\norm{z_u-z_i}\Bigr),
\]
with the convention $\sup\emptyset=-\infty$.
\end{definition}

\begin{theorem}[Certified lower bound]\label{thm:lb}
For any pool index $u$ and class $c$, for all $\tilde f\in\mathcal{F}(S)$,
\[
\tilde M_c(z_u)\ \ge\ \LB_c(u).
\]
\end{theorem}
\begin{proof}
Fix any $i\in S$ with $y_i=c$ and any $\tilde f\in\mathcal{F}(S)$.
By Lipschitzness of $\tilde M_c$ and the center constraint at $i$,
\[
\tilde M_c(z_u)\ \ge\ \tilde M_c(z_i) - L_{M,c}\norm{z_u-z_i}
\ \ge\ \underline m_i - L_{M,c}\norm{z_u-z_i}.
\]
Taking the supremum over all such $i$ yields the claim.
\end{proof}

Building an upper bound for $M_c$ at an arbitrary pool point is slightly more delicate. The key observation is that at a labeled center $i\in S$, every non-target class has a negative margin upper bound:
Lemma~\ref{lem:center_ub} shows that for any $c\neq y_i$, $\tilde M_c(z_i)\le -\underline m_i$ for all $\tilde f\in\mathcal F(S)$.

\begin{lemma}[Non-target margins are upper bounded at labeled centers]\label{lem:center_ub}
Fix $i\in S$ with label $y_i$ and lower bound $\underline m_i$.
For any $\tilde f\in\mathcal{F}(S)$ and any $c\neq y_i$,
\[
\tilde M_c(z_i)\ \le\ -\underline m_i.
\]
\end{lemma}
\begin{proof}
Fix $\tilde f\in\mathcal F(S)$. The center constraint gives
\begin{equation*}
\begin{aligned}
\tilde f_{y_i}(z_i)-\max_{k\neq y_i}\tilde f_k(z_i) &\ge \underline m_i,\\
\max_{k\neq y_i}\tilde f_k(z_i) &\le \tilde f_{y_i}(z_i)-\underline m_i.
\end{aligned}
\end{equation*}
Hence for any $c\neq y_i$, $\tilde f_c(z_i)\le \tilde f_{y_i}(z_i)-\underline m_i$.
Also, since $y_i$ is among the competitors in $\max_{k\neq c}\tilde f_k(z_i)$, we have
$\max_{k\neq c}\tilde f_k(z_i)\ge \tilde f_{y_i}(z_i)$.
Combining,
\begin{equation*}
\begin{aligned}
\tilde M_c(z_i)
&=\tilde f_c(z_i)-\max_{k\neq c}\tilde f_k(z_i)\\
&\le \bigl(\tilde f_{y_i}(z_i)-\underline m_i\bigr)-\tilde f_{y_i}(z_i)
=-\underline m_i.
\end{aligned}
\end{equation*}
\end{proof}

We therefore define the following upper envelope; Theorem~\ref{thm:ub} shows it is a certified upper bound.
\begin{definition}[Upper envelope]\label{def:ub}
For any pool index $u$ and class $c$, define
\[
\UB_c(u):=
\inf_{i\in S:\ y_i\neq c}\Bigl(-\underline m_i + L_{M,c}\norm{z_u-z_i}\Bigr),
\]
with the convention $\inf\emptyset=+\infty$.
\end{definition}

\begin{theorem}[Certified upper bound]\label{thm:ub}
For any pool index $u$ and class $c$, for all $\tilde f\in\mathcal{F}(S)$,
\[
\tilde M_c(z_u)\ \le\ \UB_c(u).
\]
\end{theorem}
\begin{proof}
If there is no $i\in S$ with $y_i\neq c$, then $\UB_c(u)=+\infty$ and the claim is trivial.
Otherwise fix any $i\in S$ with $y_i\neq c$ and any $\tilde f\in\mathcal{F}(S)$.
By Lemma~\ref{lem:center_ub}, $\tilde M_c(z_i)\le -\underline m_i$.
By Lipschitzness of $\tilde M_c$,
\[
\tilde M_c(z_u)\ \le\ \tilde M_c(z_i)+L_{M,c}\norm{z_u-z_i}
\ \le\ -\underline m_i+L_{M,c}\norm{z_u-z_i}.
\]
Taking the infimum over all such $i$ yields the claim.
\end{proof}
A direct consequence of Theorem~\ref{thm:ub} is that a negative upper envelope of a class shows that none of the members of the Lipschitz-consistent class predicts that class. 
\begin{corollary}[Certified class elimination]\label{cor:eliminate}
If $\UB_c(u)<0$, then no $\tilde f\in\mathcal{F}(S)$ can predict class $c$ at $u$.
\end{corollary}
\begin{proof}
If some $\tilde f\in\mathcal F(S)$ predicts class $c$ at $u$, then necessarily
$\tilde f_c(z_u)\ge \max_{k\neq c}\tilde f_k(z_u)$, hence $\tilde M_c(z_u)\ge 0$.
But Theorem~\ref{thm:ub} gives $\tilde M_c(z_u)\le \UB_c(u)<0$, a contradiction.
\end{proof}

\subsection{Forcing rules}\label{sec:setvalued}
In the previous section, we developed the certificate components used to reason about labels on the pool, namely the class-wise upper and lower envelopes. In this section, we first define a feasible label set using the upper envelopes, which identifies the classes that cannot be eliminated at each pool point. We then use the lower envelopes, together with control parameters, to define stronger forcing rules that can certify labels on a larger portion of the pool.

\begin{definition}[Certified feasible label set]\label{def:gamma}
Define the certified feasible label set at pool index $u$ as
\[
\Gam(u)\ :=\ \bigl\{c\in\{1,\dots,C\}:\ \UB_c(u)\ge 0\bigr\}.
\]
\end{definition}
The following theorem implies $\Gam(u)$ contains the prediction of every Lipschitz-consistent head.
\begin{theorem}[Predicted label lies in the feasible set]\label{thm:pred_in_gamma}
Fix $u\in[N]$. For any $\tilde f\in\mathcal{F}(S)$ and any $\hat y_{\tilde f}(u)\in\argmax_k \tilde f_k(z_u)$,
\[
\hat y_{\tilde f}(u)\ \in\ \Gam(u).
\]
Consequently, if $\Gam(u)=\{c^\star\}$ is a singleton, then every $\tilde f\in\mathcal{F}(S)$ predicts $c^\star$ at $u$.
\end{theorem}
\begin{proof}
Let $c^\star\in\argmax_k \tilde f_k(z_u)$. Then $\tilde M_{c^\star}(z_u)\ge 0$.
By Theorem~\ref{thm:ub}, $\tilde M_{c^\star}(z_u)\le \UB_{c^\star}(u)$, hence $\UB_{c^\star}(u)\ge 0$ and $c^\star\in\Gam(u)$.
If $\Gam(u)$ is a singleton, the second claim follows immediately.
\end{proof}

The feasible set keeps class $c$ at pool point $u$ only when $\UB_c(u)\ge 0$.
This is a strict elimination rule: any class with a negative certified upper bound is removed. To make the candidate set less aggressive, we introduce a parameter $\tau\ge 0$ and allow classes with slightly negative upper bounds to remain candidates. Larger $\tau$ means fewer classes are eliminated at this stage.

\begin{definition}[$\tau$-candidate label set]\label{def:gamma_tau}
For $\tau\ge 0$, define the relaxed candidate set
\[
\Gam_\tau(u)\ :=\ \bigl\{c\in\{1,\dots,C\}:\ \UB_c(u)\ge -\tau\bigr\}.
\]
\end{definition}

For each class $c$ at pool point $u$, the envelopes give a certified margin interval
$[\LB_c(u),\UB_c(u)]$ that holds for every head in $\mathcal F(S)$.
To force a label, it is not enough that a class could be the winner for \emph{some} consistent head; it must be the winner for \emph{all} of them.
Gap forcing enforces this by comparing the candidate class's lower bound to the other classes' upper bounds (with slack $\tau$), and by requiring a minimum amount of certified positive evidence via $\kappa$.
\begin{definition}[Certified forcing with a margin gap]\label{def:force_gap}
Fix $\tau\ge 0$ and a gap parameter $\kappa\ge 0$.
We \emph{force} a label at $u$ if there exists $c^\star$ such that
\[
\LB_{c^\star}(u)\ge \kappa
\quad\text{and}\quad
\LB_{c^\star}(u)>\max_{c\neq c^\star}\UB_c(u)+\tau.
\]
In that case output $\pi(u)=c^\star$; otherwise abstain.
\end{definition}

\begin{theorem}[Soundness of gap forcing]\label{thm:gap_sound}
If the condition in Definition~\ref{def:force_gap} holds, then every $\tilde f\in\mathcal{F}(S)$ predicts $c^\star$ at $u$.
\end{theorem}
\begin{proof}
Fix any $\tilde f\in\mathcal F(S)$.
By Theorem~\ref{thm:lb}, $\tilde M_{c^\star}(z_u)\ge \LB_{c^\star}(u)$.
By Theorem~\ref{thm:ub}, for every $c\neq c^\star$ we have $\tilde M_c(z_u)\le \UB_c(u)$.
Thus,
\[
\tilde M_{c^\star}(z_u)
\ \ge\ \LB_{c^\star}(u)
\ >\ \max_{c\neq c^\star}\UB_c(u)+\tau
\ \ge\ \max_{c\neq c^\star}\tilde M_c(z_u).
\]
Hence $c^\star$ is the \emph{unique} maximizer of the margins $\{\tilde M_c(z_u)\}_{c=1}^C$.

Now let $\hat y\in\argmax_k \tilde f_k(z_u)$ be any logit maximizer of $\tilde f$ at $u$.
Then $\tilde M_{\hat y}(z_u)=\tilde f_{\hat y}(z_u)-\max_{k\neq \hat y}\tilde f_k(z_u)\ge 0$, so $\hat y$ is also a maximizer of the margins.
Since $c^\star$ is the unique margin maximizer, we must have $\hat y=c^\star$.
\end{proof}

In summary, we can force a label (i.e., certify agreement over $\mathcal F(S)$) in two ways.
\textbf{(i) Singleton forcing:} if the relaxed candidate set $\Gam_\tau(u)$ is a singleton, then all other classes are eliminated by a negative upper envelope.
\textbf{(ii) Gap forcing:} even when multiple classes remain, we can force $c^\star$ if it has sufficient certified positive evidence and is separated from all competitors:
\[
\LB_{c^\star}(u)\ge \kappa
\quad\text{and}\quad
\LB_{c^\star}(u)>\max_{c\neq c^\star}\UB_c(u)+\tau.
\]
The parameter $\tau$ controls how aggressively we eliminate classes via upper bounds, while $\kappa$ enforces a minimum amount of certified positive evidence for the forced class.

\section{Properties of the Forcing Rules}\label{sec:practical}

In this short section, we state a property of forcing rules ( introduced in ~\ref{sec:setvalued}): adding points with forced label to our initial labeled set does not change the Lipschitz-consistent version space. For this, we formally define the set of all forced points:

\begin{definition}[Forced pseudo-label set]\label{def:pseudolabel_set}
Fix thresholds $(\tau,\kappa)$ and a labeled set $S$.
Let $\pi_{S,\tau,\kappa}$ be the forcing rule from Section~\ref{sec:setvalued}.
Define the forced pseudo-label set
\[
P_{S,\tau,\kappa} \ :=\ \{u\in[N]: \pi_{S,\tau,\kappa}(u)\neq \bot\},
\]
and the induced pseudo-labels $\hat y(u):=\pi_{S,\tau,\kappa}(u)$ for $u\in P_{S,\tau,\kappa}$.
\end{definition}

\begin{theorem}[Version-space closure under forced pseudo-labels]\label{thm:closure}
Fix $(\tau,\kappa)$ and a labeled set $S$ with bounds $\{\underline m_i\}_{i\in S}$ and Lipschitz constants $\{L_{M,c}\}_{c=1}^C$.
Let $u\in P_{S,\tau,\kappa}$ with forced label $\hat y(u)=c$.
Define $S':=S\cup\{u\}$ with label $y_u:=c$ and center bound
\[
\underline m_u :=
\begin{cases}
0, & \text{if $u$ is forced via singleton forcing},\\
\kappa, & \text{if $u$ is forced via gap forcing}.
\end{cases}
\]
Then $\mathcal{F}(S')=\mathcal{F}(S)$.
\end{theorem}

\begin{proof}
By construction, $S'$ imposes all constraints of $S$ plus one additional center constraint at $u$.
Hence $\mathcal F(S')\subseteq \mathcal F(S)$.
It remains to show $\mathcal F(S)\subseteq \mathcal F(S')$.

Fix any $\tilde f\in\mathcal F(S)$. We must show $\tilde M_{c}(z_u)\ge \underline m_u$, where $c=\hat y(u)$.

\smallskip
\noindent\emph{Case 1: $u$ is forced via singleton forcing.}
By the singleton-forcing branch of the rule, $\Gam_\tau(u)=\{c\}$.
Equivalently, every competitor $c'\neq c$ satisfies $\UB_{c'}(u)<-\tau\le 0$, so by Corollary~\ref{cor:eliminate}
no $\tilde f\in\mathcal F(S)$ can predict any $c'\neq c$ at $u$.
Therefore any prediction at $u$ must be $c$, which implies $\tilde M_c(z_u)\ge 0=\underline m_u$.

\smallskip
\noindent\emph{Case 2: $u$ is forced via gap forcing.}
Gap forcing implies $\LB_c(u)\ge \kappa=\underline m_u$. By Theorem~\ref{thm:lb},
$\tilde M_c(z_u)\ge \LB_c(u)\ge \underline m_u$.

\smallskip
Thus, in both cases $\tilde f$ satisfies the added center constraint at $u$, so $\tilde f\in\mathcal F(S')$.
Therefore, $\mathcal F(S)\subseteq \mathcal F(S')$, and $\mathcal F(S')=\mathcal F(S)$.
\end{proof}

In other words, adding forced pseudo-labels as new centers does not strengthen the consistency constraints: the version space (and hence the semantics of the certificates) remains unchanged. This can be useful for iterative certified label expansion / pseudo-labeling.

\section{Budgeted Querying via a Submodular Proxy}\label{sec:budgeted}

We now consider budgeted acquisition: given a labeling budget $k$, which pool points should be queried to maximize the portion of the pool that becomes certifiably decidable (forced)? Since certified coverage depends on unknown labels and on head-dependent quantities learned after querying, we introduce a simple label-free geometric proxy motivated by the locality of certificate propagation, and optimize it with a greedy algorithm with a standard $(1-1/e)$ guarantee.

\subsection{Margin Floor and a Uniform Certified Radius}\label{sec:marginfloor}

This subsection extracts a geometric quantity from the certificate: a uniform radius around each labeled center within which the center's class has certified positive evidence (via the lower envelope).
This radius does \emph{not} by itself force a label, since elimination of competing classes requires upper-envelope conditions.
Its purpose is to connect certification to a ball-coverage proxy for budgeted querying.

To derive such a radius, we impose a mild margin-floor assumption on labeled centers.

\begin{assumption}[Center margin floor]\label{ass:marginfloor}
There exists $\gamma>0$ such that for all labeled centers $i\in S$, $\underline m_i \ge \gamma$.
\end{assumption}

Let
\[
L_{\max}:=\max_{c\in\{1,\dots,C\}} L_{M,c},
\qquad
\rhocert := \frac{\gamma}{L_{\max}}.
\]
Under this assumption, every labeled center certifies positive evidence inside a fixed radius $\rho_{\mathrm{cert}}$.
\begin{lemma}[Uniform positive-evidence propagation]\label{lem:uniform}
Under Assumption~\ref{ass:marginfloor}, for any labeled center $i\in S$ and any pool point $u\in[N]$ with
$\norm{z_u-z_i}<\rhocert$, we have $\LB_{y_i}(u)>0$.
\end{lemma}

\begin{proof}
Fix any labeled center $i\in S$. By Assumption~\ref{ass:marginfloor}, $\underline m_i\ge \gamma$.
If $\|z_u-z_i\|<\rhocert=\gamma/L_{\max}$, then
\[
\underline m_i - L_{M,y_i}\|z_u-z_i\|
\ \ge\
\gamma - L_{\max}\|z_u-z_i\|
\ >\ 0.
\]
Since $\LB_{y_i}(u)$ is the supremum over all class-$y_i$ centers (Definition~\ref{def:lb}),
it is at least the contribution of center $i$, hence $\LB_{y_i}(u)>0$.
\end{proof}

\subsection{Budgeted Subset Selection Objectives}\label{sec:coverage}

When labels are acquired under a finite budget, selecting which pool points to query is a one-shot subset selection problem.
In our setting, the downstream utility of a labeled set $S$ is measured by how much of the pool becomes \emph{certifiably}
decidable (forced) under the envelopes derived from $S$.

\subsubsection{Ideal (label-dependent) objectives}

Let $\pi_{S}$ denote the certificate-driven forcing rule induced by a labeled set $S$
(e.g., the rule in Section~\ref{sec:setvalued} with fixed $(\tau,\kappa)$).
A natural goal is to maximize certified coverage:
\[
\max_{S\subseteq[N],\ |S|=k}\ \Cov(\pi_{S}).
\]
More refined objectives include shrinking ambiguity via candidate-set sizes, e.g.
\[
\max_{S\subseteq[N],\ |S|=k}\ -\frac{1}{N}\sum_{u\in[N]}|\Gam_\tau^{(S)}(u)|,
\]
where $\Gam_\tau^{(S)}(u)$ emphasizes that candidate sets depend on which indices are labeled.

These objectives are generally \emph{label-dependent}: the envelopes depend on the labels revealed by querying $S$ and on head-dependent quantities
(margins and Lipschitz constants) learned from those labels. This makes direct optimization of the ideal objective intractable prior to querying.

\subsubsection{Proxy for Geometric Coverage}

While the above objectives are desirable, they involve the labels and number of certificates that are only accessible once the oracle is queried and the certification head is fitted. In order to derive a pre-query acquisition function, we can use only the part of the certification mechanism which is accessible based on the geometry of the unlabeled pool: \emph{locality of propagation} in the embedding space. As an intuition, lower-envelope certification spreads evidence from labeled centers to points in their vicinity, with a strength that decays with distance. While the labels are not yet accessible, and we cannot yet know \emph{which} class will propagate or whether upper-envelope elimination will eventually enforce a label, we can still prefer query sets whose neighborhoods collectively cover large parts of the pool. A natural label-free geometric proxy is thus given by unions of radius-$\rhoacq$ balls centered at the acquired points.

Let us fix some $\rhoacq>0$. Then the covered set is given by, with the objective to maximize:
\begin{equation*}
\begin{aligned}
\COVset_{\rhoacq}(S)
&:= \bigl\{u\in[N]: \exists\, i\in S \text{ such that } \norm{z_u-z_i}<\rhoacq \bigr\},\\
F_{\rhoacq}(S)
&:= \bigl|\COVset_{\rhoacq}(S)\bigr|.
\end{aligned}
\end{equation*}

Equivalently, $F_{\rhoacq}(S)$ counts how many pool points lie within distance $\rhoacq$ of at least one selected point.

Under the margin-floor assumption (Assumption~\ref{ass:marginfloor}), setting $\rhoacq=\rhocert$ gives a natural certificate-based interpretation of geometric coverage. Any pool point covered by the proxy lies within distance $\rhocert$ of at least one labeled center. By Lemma~\ref{lem:uniform}, this implies that the point receives \emph{certified positive evidence} for the label of that center (i.e., $\LB_{y_i}(u)>0$ for some $i\in S$). Thus, $\COVset_{\rhocert}(S)$ can be interpreted as the portion of the pool where the certificate guarantees local positive evidence for at least one class.

\paragraph{Scope of the proxy.}
The proxy $F_{\rhoacq}$ captures only the geometric aspect of certification. It does not account for (i) the (currently) unknown class labels of queried points, (ii) upper-envelope elimination and competition among classes, or (iii) other head-dependent effects. Consequently, $F_{\rhoacq}$ is generally not equal to certified coverage. Its purpose is instead to provide a simple pre-query surrogate that captures the locality structure of propagation and admits optimization guarantees. In the next section, we show that $F_{\rhoacq}$ is a monotone submodular function, so greedy optimization achieves the standard $(1-1/e)$ approximation guarantee.
\subsection{Submodularity and Greedy Approximation}\label{sec:submod}

We now show that the geometric proxy
\[
F_{\rhoacq}(S)=|\COVset_{\rhoacq}(S)|
\]
admits efficient optimization under a labeling budget. The key observation is that $F_{\rhoacq}$ is a standard coverage objective (the size of a union of neighborhoods), which is monotone and submodular. Therefore, under the cardinality constraint $|S|\le k$, the greedy algorithm achieves the classical $(1-1/e)$ approximation guarantee.

Algorithm~\ref{alg:greedy} implements exact greedy maximization of $F_{\rhoacq}$ by repeatedly selecting the point with the largest marginal increase in uncovered pool points.

\begin{algorithm}[t]
\caption{Greedy Coverage Subset Selection (budget $k$)}\label{alg:greedy}
\begin{algorithmic}[1]
\State \textbf{Input:} candidate indices $[N]$, embeddings $\{z_i\}_{i=1}^N$, radius $\rhoacq$, budget $k$
\State $S \gets \emptyset$
\State $\mathsf{Covered} \gets \emptyset$
\For{$t=1$ \textbf{to} $k$}
  \ForAll{$x \in [N] \setminus S$}
    \State $B_x \gets \{u\in[N]:\norm{z_u-z_x}<\rhoacq\}$
    \State $\Delta(x) \gets |B_x \setminus \mathsf{Covered}|$
  \EndFor
  \State Choose $x^\star \in \argmax_{x\in [N]\setminus S}\Delta(x)$
  \State $S \gets S \cup \{x^\star\}$
  \State $\mathsf{Covered} \gets \mathsf{Covered} \cup B_{x^\star}$
\EndFor
\State \textbf{Return:} $S$
\end{algorithmic}
\end{algorithm}

\begin{theorem}[Monotone submodularity and greedy approximation]\label{thm:greedy}
The set function $F_{\rhoacq}$ is nonnegative, monotone, and submodular. Let $S^\star$ be an optimal size-$k$ set maximizing $F_{\rhoacq}$, and let $S_k$ be the set returned by Algorithm~\ref{alg:greedy}. Then
\[
F_{\rhoacq}(S_k)\ \ge\ \left(1-\frac{1}{e}\right)F_{\rhoacq}(S^\star).
\]
\end{theorem}
\begin{proof}[Proof sketch]
For each candidate center $i\in[N]$, define its radius-$\rhoacq$ neighborhood
\[
B_i := \{u\in[N]: \norm{z_u-z_i}<\rhoacq\}.
\]
Then
\[
F_{\rhoacq}(S)=\left|\bigcup_{i\in S} B_i\right|,
\]
so $F_{\rhoacq}$ is a standard set-coverage function.

Monotonicity follows since adding centers can only enlarge the union. Submodularity follows by diminishing returns: for $A\subseteq B$ and $x\notin B$, the marginal gain of adding $x$ is the number of newly covered points in $B_x$, which can only decrease as the covered set grows. The greedy $(1-1/e)$ guarantee then follows from the classical Nemhauser--Wolsey--Fisher result for monotone submodular maximization under a cardinality constraint \cite{nemhauser1978analysis}. Full details are provided in the supplementary material.
\end{proof}

Our guarantees establish agreement on $\mathcal F(S)$ over the fixed pool. Under additional assumptions (e.g., class separation/purity in the embedding space), this agreement can also imply correctness (formal statements in the supplementary material).

\section{Experiments}
\label{sec:experiments}

We evaluate whether our certificate mechanism (i) yields meaningful forced coverage under small label budgets,
(ii) improves risk--coverage (RC) tradeoffs \emph{on the certifiable domain}, and (iii) benefits from budgeted acquisition via the geometric proxy
(Sections~\ref{sec:coverage}--\ref{sec:submod}).

\paragraph{Protocol (fixed-pool).}
We use a transductive setting where the unlabeled pool is the test split: \textsc{CIFAR-10}~\cite{krizhevsky2009cifar} ($N=10{,}000$) and \textsc{SVHN}~\cite{netzer2011reading} ($N=26{,}032$), both with $C=10$ classes.
We embed pool points with a ResNet-18 encoder~\cite{he2016deep} trained on the training split and use $\ell_2$-normalized penultimate-layer features ($d=512$).
Additional training and preprocessing details are in the supplementary material.

\paragraph{Budgeted acquisition.}
For budgets $b\in\{0.5\%,1\%,2\%,5\%\}$ (i.e., $k=\lceil bN\rceil$), we compare three pool-based acquisition strategies: greedy ball coverage (\texttt{greedy}), farthest-first $k$-center (\texttt{kcenter}) \cite{sener2018active}, and uniform random (\texttt{random}). Radius selection and graph-construction details are deferred to the supplementary material.

\paragraph{Certificates and baselines.}
Given labeled set $S$, we compute the envelopes $\LB_c(u),\UB_c(u)$ (Definitions~\ref{def:lb}--\ref{def:ub}) and apply \texttt{cert\_full}: singleton forcing from $\Gam_\tau(u)$ (Definition~\ref{def:gamma_tau}) and, otherwise, gap forcing (Definition~\ref{def:force_gap}), with $\kappa=0$ unless stated. We compare against post-hoc selective baselines on the same trained head: \texttt{softmax} thresholding, \texttt{margin} (top-1/top-2 logit gap) thresholding, APS conformal singletons (\texttt{conformal\_aps}), and a lightweight gating model (\texttt{selectivenet\_gate}). RC curves are obtained by sweeping each method's control parameter; implementation details are in the supplementary material.

\paragraph{Metrics.}
We report coverage and selective risk (Definition~\ref{def:selective}) and summarize RC curves via AURC. Because methods can attain different maximum coverages, we also report truncated AURC over $[0,\mathrm{cov}^{\max}_{\mathrm{cert}}]$, where $\mathrm{cov}^{\max}_{\mathrm{cert}}$ is the maximum coverage of \texttt{cert\_full} in the same (dataset, budget, acquisition) setting. Table~\ref{tab:covmax:main} reports these certified coverage ceilings, which define the truncation endpoint for matched-coverage comparisons. The values show strong dependence on acquisition geometry at low budgets (especially \texttt{greedy} vs.\ \texttt{kcenter}). We note that $\mathrm{cov}^{\max}_{\mathrm{cert}}$ need not be monotone in budget in our pipeline, since changing $S$ also changes the trained head, margin lower bounds, and Lipschitz constants used by the certificate. Additional quantitative results are provided in the supplementary material.

\begin{table}[t]
\centering
\caption{Maximum certified coverage $\mathrm{cov}^{\max}_{\mathrm{cert}}$ achieved by \texttt{cert\_full} across acquisitions and budgets (test pools). These values define the truncation endpoint for truncated AURC comparisons in the corresponding setting.}
\label{tab:covmax:main}
\scriptsize
\setlength{\tabcolsep}{3.8pt}
\renewcommand{\arraystretch}{0.96}
\begin{tabular}{llcccc}
\toprule
Dataset & Acquisition & 0.5\% & 1\% & 2\% & 5\% \\
\midrule
\multirow{3}{*}{CIFAR-10}
& \texttt{greedy}  & 0.8042 & 0.8250 & 0.8079 & 0.8527 \\
& \texttt{random}  & 0.5145 & 0.6922 & 0.7490 & 0.8534 \\
& \texttt{kcenter} & 0.1385 & 0.2130 & 0.2968 & 0.4213 \\
\midrule
\multirow{3}{*}{SVHN}
& \texttt{greedy}  & 0.7787 & 0.7207 & 0.6818 & 0.6076 \\
& \texttt{random}  & 0.6820 & 0.6263 & 0.7792 & 0.8437 \\
& \texttt{kcenter} & 0.3356 & 0.3972 & 0.5676 & 0.5795 \\
\bottomrule
\end{tabular}
\end{table}

\paragraph{Key findings.}
(1) \textbf{Competitive RC tradeoffs on the certified domain.}
On both datasets, \texttt{cert\_full} is competitive with (and often improves upon) post-hoc thresholding baselines when comparisons are restricted to the
certifiable domain $[0,\mathrm{cov}^{\max}_{\mathrm{cert}}]$.
(2) \textbf{Certified coverage is acquisition-sensitive.}
Greedy ball coverage frequently yields substantially larger certified domains than \texttt{kcenter} and \texttt{random} at the same budget (Table~\ref{tab:covmax:main}), supporting the geometric proxy motivation.
(3) \textbf{Conservative-by-design coverage ceiling.}
Because our method certifies \emph{agreement} over a Lipschitz-consistent version space, it is intentionally conservative: in several settings, \texttt{cert\_full} attains low selective risk on the certified domain but does not extend to the largest coverages reached by post-hoc thresholding methods. We view this as a semantic tradeoff (agreement-based reliability vs.\ attainable coverage), and Figure~\ref{fig:cifar10:rc_b002} together with the supplementary results makes this behavior explicit.

Figure~\ref{fig:cifar10:rc_b002} illustrates representative RC frontiers on CIFAR-10 at budget $b=2\%$ for the three acquisition strategies.
\begin{figure}[t]
\centering
\begin{subfigure}[t]{0.32\linewidth}
  \centering
  \includegraphics[width=\linewidth]{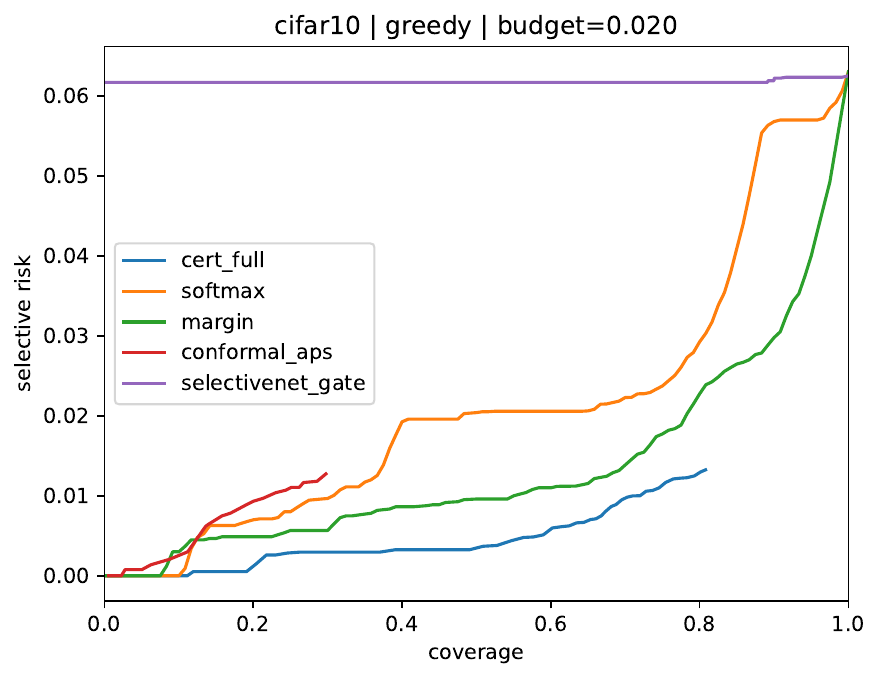}
  \caption{\texttt{greedy}}
\end{subfigure}
\begin{subfigure}[t]{0.32\linewidth}
  \centering
  \includegraphics[width=\linewidth]{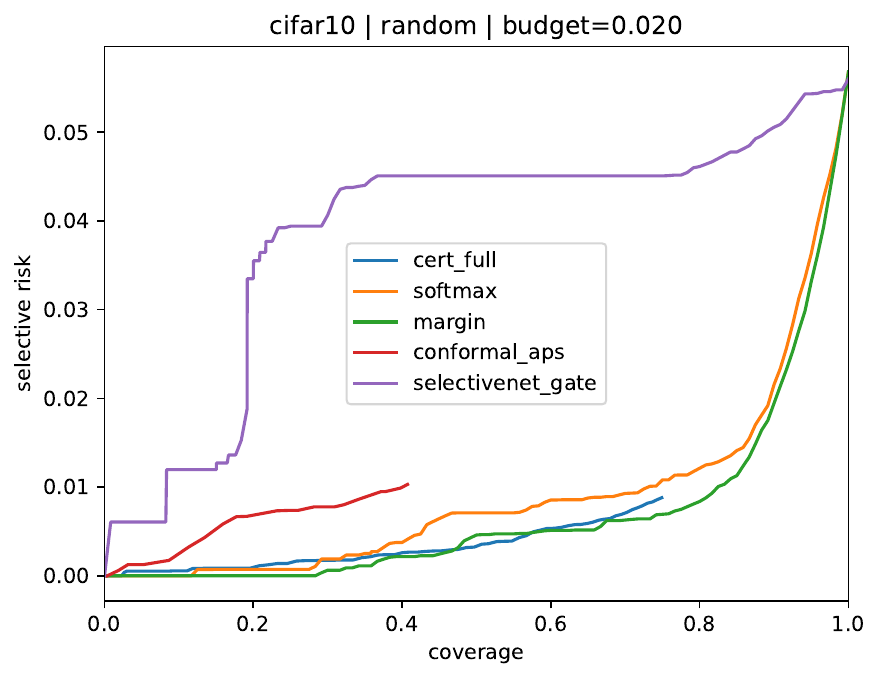}
  \caption{\texttt{random}}
\end{subfigure}
\begin{subfigure}[t]{0.32\linewidth}
  \centering
  \includegraphics[width=\linewidth]{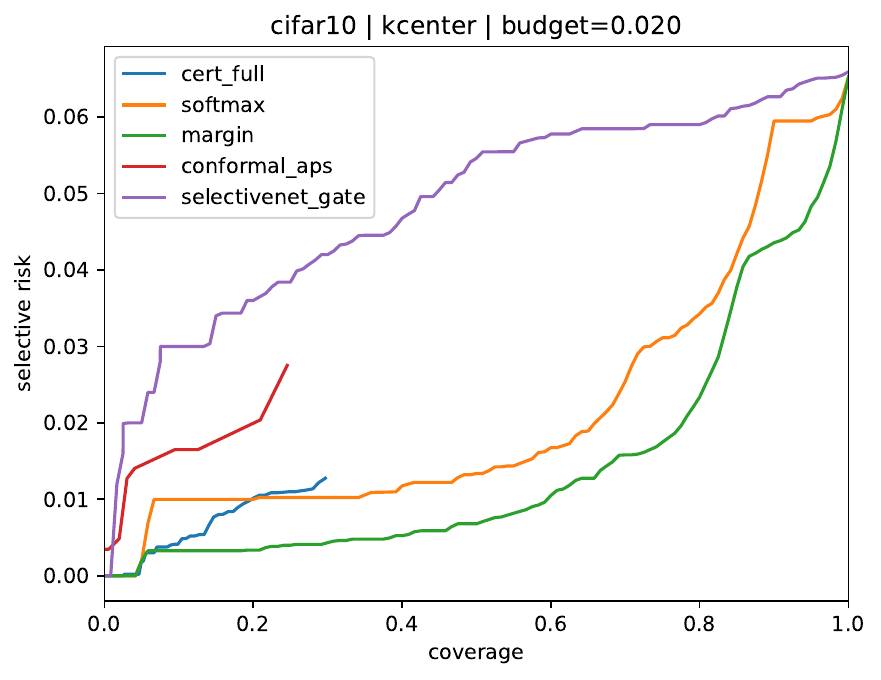}
  \caption{\texttt{kcenter}}
\end{subfigure}
\caption{CIFAR-10 RC frontiers at budget $b=2\%$ ($k=200$). Comparisons are most meaningful on the certified domain $[0,\mathrm{cov}^{\max}_{\mathrm{cert}}]$. Additional plots and quantitative summaries are in the supplementary material.}
\label{fig:cifar10:rc_b002}
\end{figure}

\section{Conclusion}\label{sec:conclusion}
We introduced a certification-based view of selective prediction in the fixed-pool setting, where abstention is treated as lack of forced agreement rather than confidence thresholding. We predict only when observed labels and Lipschitz margin constraints force agreement over a Lipschitz-consistent version space of heads.

We derived two-sided Lipschitz margin envelopes leading to certified feasible label sets and sound force-or-abstain rules, proved a closure property for reusing forced pseudo-labels without changing the version-space semantics, and proposed a geometric proxy for budgeted querying with a monotone submodular objective and a greedy $(1-1/e)$ guarantee.

Experiments show non-trivial certified coverage and that proxy-based acquisition can substantially expand the certified region under a fixed labeling budget. A practical limitation is conservativeness: the method often attains low selective risk on the certified domain but may saturate at lower maximum coverage than less constrained post-hoc abstention rules.

Future work includes weaker and more realistic assumptions linking agreement to correctness, and relaxed (e.g., sampled) variants of forcing that increase coverage while retaining interpretability.

\section*{Generative AI use disclosure}
We used ChatGPT to assist with writing and presentation (grammar/style edits, rephrasing, and improving the wording of definitions and theorem statements).
For Section~\ref{sec:submod}, we also used it to obtain a proof hint/outline and to refine the exposition.
All technical content and claims, including proof correctness, were verified by the authors, who take full responsibility for the manuscript and any copyright/IP issues.


\begin{thebibliography}{18}



\bibitem{chow1970reject}
Chow, C.K.: On Optimum Recognition Error and Reject Tradeoff. IEEE Transactions on Information Theory \textbf{16}(1), 41--46 (1970). \doi{10.1109/TIT.1970.1054406}

\bibitem{elyaniv2010selective}
El-Yaniv, R., Wiener, Y.: On the Foundations of Noise-Free Selective Classification. Journal of Machine Learning Research \textbf{11}, 1605--1641 (2010)

\bibitem{geifman2019selectivenet}
Geifman, Y., El-Yaniv, R.: SelectiveNet: A Deep Neural Network with an Integrated Reject Option. In: Proceedings of the 36th International Conference on Machine Learning. Proceedings of Machine Learning Research, vol.~97, pp.~2151--2159. PMLR (2019)


\bibitem{nemhauser1978analysis}
Nemhauser, G.L., Wolsey, L.A., Fisher, M.L.: An Analysis of Approximations for Maximizing Submodular Set Functions---I. Mathematical Programming \textbf{14}(1), 265--294 (1978). \doi{10.1007/BF01588971}






\bibitem{sener2018active}
Sener, O., Savarese, S.: Active Learning for Convolutional Neural Networks: A Core-Set Approach. In: International Conference on Learning Representations (2018). arXiv:1708.00489


\bibitem{bartlett2017spectral}
Bartlett, P.L., Foster, D.J., Telgarsky, M.: Spectrally-Normalized Margin Bounds for Neural Networks. In: Advances in Neural Information Processing Systems 30, pp.~6240--6249. Curran Associates, Inc. (2017)

\bibitem{miyato2018spectral}
Miyato, T., Kataoka, T., Koyama, M., Yoshida, Y.: Spectral Normalization for Generative Adversarial Networks. In: International Conference on Learning Representations (2018). arXiv:1802.05957

\bibitem{tsuzuku2018lmt}
Tsuzuku, Y., Sato, I., Sugiyama, M.: Lipschitz-Margin Training: Scalable Certification of Perturbation Invariance for Deep Neural Networks. In: Advances in Neural Information Processing Systems (2018). arXiv:1802.04034

\bibitem{krizhevsky2009cifar}
Krizhevsky, A.: Learning Multiple Layers of Features from Tiny Images. Technical Report, University of Toronto (2009)

\bibitem{zhu2003harmonic}
Zhu, X., Ghahramani, Z., Lafferty, J.: Semi-supervised Learning Using Gaussian Fields and Harmonic Functions. In: Proceedings of the 20th International Conference on Machine Learning, pp.~912--919 (2003)

\bibitem{zhou2004localglobal}
Zhou, D., Bousquet, O., Lal, T.N., Weston, J., Sch{\"o}lkopf, B.: Learning with Local and Global Consistency. In: Advances in Neural Information Processing Systems 16, pp.~321--328 (2004)


\bibitem{chapelle2006ssl}
Chapelle, O., Sch{\"o}lkopf, B., Zien, A. (eds.): Semi-Supervised Learning. MIT Press (2006)




\bibitem{vovk2005algorithmic}
Vovk, V., Gammerman, A., Shafer, G.: Algorithmic Learning in a Random World. Springer (2005)

\bibitem{romano2020classification}
Romano, Y., Sesia, M., Cand{\`e}s, E.J.: Classification with Valid and Adaptive Coverage. In: Advances in Neural Information Processing Systems (2020)

\bibitem{angelopoulos2021uncertaintysets}
Angelopoulos, A.N., Bates, S., Malik, J., Jordan, M.I.: Uncertainty Sets for Image Classifiers using Conformal Prediction. In: International Conference on Learning Representations (2021). arXiv:2009.14193


\bibitem{geifman2017selective}
Geifman, Y., El-Yaniv, R.: Selective Classification for Deep Neural Networks. In: Advances in Neural Information Processing Systems, vol.~30 (2017)


\bibitem{netzer2011reading}
Netzer, Y., Wang, T., Coates, A., Bissacco, A., Wu, B., Ng, A.Y.: Reading Digits in Natural Images with Unsupervised Feature Learning. In: NIPS Workshop on Deep Learning and Unsupervised Feature Learning, vol.~2011(2), p.~5 (2011)







\bibitem{he2016deep}
He, K., Zhang, X., Ren, S., Sun, J.: Deep Residual Learning for Image Recognition.
In: Proceedings of the IEEE Conference on Computer Vision and Pattern Recognition (CVPR),
pp.~770--778 (2016)

\bibitem{kottke2022stopping}
Kottke, D., Sandrock, C., Krempl, G., Sick, B.:
A Stopping Criterion for Transductive Active Learning.
In: Machine Learning and Knowledge Discovery in Databases (ECML PKDD 2022), pp.~468--484. Springer (2022).
\doi{10.1007/978-3-031-26412-2_29}

\bibitem{goldwasser2020beyond}
Goldwasser, S., Kalai, A.T., Kalai, Y.T., Montasser, O.:
Beyond Perturbations: Learning Guarantees with Arbitrary Adversarial Test Examples.
In: Advances in Neural Information Processing Systems (NeurIPS) 33 (2020)

\bibitem{kalai2021towards}
Kalai, A.T., Kanade, V.:
Towards Optimally Abstaining from Prediction with OOD Test Examples.
In: Advances in Neural Information Processing Systems (NeurIPS) 34, pp.~12774--12785 (2021)
\end{thebibliography}
\end{document}